\documentclass{article}

     \usepackage[preprint]{neurips_2020}

\usepackage[utf8]{inputenc} 
\usepackage[T1]{fontenc}    
\usepackage{hyperref}       
\usepackage{url}            
\usepackage{booktabs}       
\usepackage{amsfonts}       
\usepackage{nicefrac}       
\usepackage{microtype}      
\usepackage{enumitem}

\usepackage{forest}
\usepackage[english]{babel}
\usepackage{wrapfig}
\usepackage{graphicx}
\usepackage{xfrac}
\usepackage{tikz,pgfplots}
\pgfplotsset{compat=newest}
\usepackage[utf8]{inputenc}
\usetikzlibrary{patterns}
\usepackage{amsmath}
\usepackage{amssymb}
\usepackage{colortbl}
\usepackage{cancel}
\usepackage{ulem}
\usepackage{multirow}
\usepackage{relsize}
\usepackage{algorithm}
\usepackage{algorithmic}
\usepackage{forloop}
\usepackage{bm}
\usepackage{mathtools}
\newcounter{loopcntr}
\usepackage{subfigure}
\usepackage{booktabs}
\usepackage{times}
\usepackage{etoolbox}
\usepackage{amsthm}
\usepackage{enumitem}

\patchcmd{\thebibliography}{\section*{\refname}}{}{}{}
\let\emph\relax 
\DeclareTextFontCommand{\emph}{\it}

\newtheorem{theorem}{Theorem}
\newtheorem{proposition}{Proposition}
\newtheorem{example}{Example}
\newtheorem{corollary}{Corollary}

\title{Reparameterizing Mirror Descent \\as Gradient Descent}

\author{ {Ehsan Amid\thanks{An earlier version of this manuscript (with additional results on the matrix case) appeared as "Interpolating Between Gradient Descent and Exponentiated Gradient Using Reparameterized Gradient Descent".}\; and Manfred K. Warmuth} \\
Google Research, Brain Team\\
Mountain View, CA \\
\texttt{\{eamid, manfred\}@google.com}
}

\newcommand{\simplex}{\Delta}

\renewcommand{\v}{\bm{v}}

\usepackage{mathtools}

\newcommand{\Circ}{\!\circ\!}

\newcommand{\zero}{\bm{0}}

\DeclareMathOperator{\range}{range}
\DeclareMathOperator{\domain}{dom}
\DeclareMathOperator{\sign}{sign}
\DeclareMathOperator{\tr}{tr}

\renewcommand{\S}{\bm{S}}

\newcommand{\matlogt}{{\text{\bf log}}_\tau}

\newcommand{\matexpt}{{\text{\bf exp}}_\tau}

\newcommand{\matsign}{{\text{\bf sign}}}

\newcommand{\del}{\bm{\delta}}
\newcommand{\y}{\bm{y}}

\DeclareMathOperator{\diag}{diag}

\newcommand\sbullet[1][.5]{\mathbin{\vcenter{\hbox{\scalebox{#1}{$\bullet$}}}}}
\renewcommand{\dot}[1]{    {\stackrel{\scriptscriptstyle \sbullet[0.42]}{#1}}    }
\newcommand{\wstardot}{%
  \mathrel{\vbox{\offinterlineskip\ialign{%
    \hfil##\hfil\cr
    $\scriptscriptstyle\sbullet[0.4]$\cr
    \noalign{\kern+.4ex}
    $\w$\cr
}}}^*\!\!(t)}

\DeclareMathOperator*{\argmin}{\mathop{\mathrm{argmin}}}

\DeclareMathOperator*{\argsup}{\mathop{\mathrm{argsup}}}

\newcommand{\half}{\sfrac12}

\renewcommand{\u}{\bm{u}}

\newcommand{\x}{\bm{x}}

\renewcommand{\H}{\bm{H}}

\newcommand{\X}{\bm{X}}

\newcommand{\J}{\bm{J}}
\newcommand{\I}{\bm{I}}
\renewcommand{\o}{{\bm{w}}}

\renewcommand{\v}{\bm{v}}

\newcommand{\otil}{{\bm{\widetilde{\w}}}}
\newcommand{\ot}{\widetilde{\o}}
\newcommand{\ott}{\widetilde{w}}

\newcommand{\dd}{\mathrm{d}}

\newcommand{\lam}{\bm{\lambda}}
\newcommand{\one}{\bm{1}}
\newcommand{\bnu}{\bm{\nu}}

\newcommand{\RR}{\mathbb{R}}
\newcommand{\dom}{\mathcal{C}}
\newcommand{\U}{\bm{U}}
\renewcommand{\P}{\bm{P}}
\newcommand{\W}{\bm{W}}
\newcommand{\Wt}{\widetilde{\W}}

\newcommand{\w}{{\bm{w}}
}
\newcommand{\D}{\bm{D}}
\newcommand{\C}{\mathcal{C}}

\newcommand{\V}{\bm{V}}

\newcommand{\Red}[1]{\color{red}{#1}\color{black}}
\newcommand{\Blue}[1]{\color{blue}{#1}\color{black}}

\newcommand{\Magenta} [1]{{\color{magenta} {#1}}}

\definecolor{darkgreen}{rgb}{0.09, 0.45, 0.27}

\begin{document}

\maketitle
\begin{abstract}
    Most of the recent successful applications of neural networks have been based on training with gradient descent updates. However, for some small networks, other mirror descent updates learn provably more efficiently when the target is sparse. We present a general framework for casting a mirror descent update as a gradient descent update on a different set of parameters. In some cases, the mirror descent reparameterization can be described as training a modified network with standard backpropagation. The reparameterization framework is versatile and covers a wide range of mirror descent updates, even cases where the domain is constrained. Our construction for the reparameterization argument is done for the continuous versions of the updates. Finding general criteria for the discrete versions to closely track their continuous counterparts remains an interesting open problem.
\end{abstract}

\section{Introduction}
Mirror descent (MD)~\citep{mirror,eg} refers to a family of updates which transform the parameters $\o\in \dom$ from a convex domain $\dom \in \RR^d$ via a {\it link} function (a.k.a. mirror map) $f:\, \dom \rightarrow \RR^d$ before applying the descent step. 
The {\it continuous-time mirror descent} (CMD) update, which can be seen as the limit case of (discrete-time) MD, corresponds to the solution of the
following ordinary differential equation (ODE)~\citep{mirror,jagota,raginsky}:
    \begin{align}
        \label{eq:cont-md}
	\frac{f(\o(t+h))-f(\o(t))}h
	\;&\stackrel{h\rightarrow 0}{=}\;
	\dot{f}\big(\o(t)\big)=-\eta\,\nabla L(\o(t))\, ,
	 && \text{(CMD)}
	\\\o(t+1)&\;\,=\;\,f^{-1}\Big(f(\o(t))-\eta \nabla
	L(\w(t))\Big)\,.
	 && \text{(MD)}
    \end{align}
    Here $\dot{f} \coloneqq \frac{\partial f}{\partial t}$
    is the time derivative of the link function and the
    vanilla discretized MD update is obtained by setting
    the step size $h$ equal to 1. The main link functions
    investigated in the past are $f(\o)=\o$ and
    $f(\o)=\log(\o)$ leading to the gradient descent (GD)
    and the unnormalized exponentiated gradient (EGU)
    family of updates\footnote{The normalized version is
    called EG and the two-sided version EGU$^\pm$.  More
    about this later.}. These two link functions are
    associated with the squared Euclidean and the relative
    entropy divergences, respectively. For example, the
    classical Perceptron and Winnow algorithms are
    motivated using the identity and log links,
    respectively, when the loss is the hinge loss. A number
    of papers discuss the difference between the two
    updates~\citep{eg,pnorm,nie,eladEGpm} and their
    rotational invariance properties have been explored
    in~\citep{rotinv}. In particular, the {\it Hadamard
    problem} is a paradigmatic linear problem that shows
    that EGU can converge dramatically faster than GD when
    the instances are dense and the target weight vector is
    sparse~\citep{perceptwin,span}. This property is linked
    to the strong-convexity of the relative entropy w.r.t.
    the $\mathrm{L}_1$-norm\footnote{Whereas the squared
    Euclidean divergence (which motivates GD) is strongly-convex w.r.t. the
    $\mathrm{L}_2$-norm.}~\citep{shay}, which motivates the
    discrete EGU update.

\paragraph{Contributions}
Although other MD updates can be drastically more efficient
than GD updates on certain classes of problems,
it was assumed that such MD updates are not realizable using GD.
In this paper, we show that in fact a large number of MD updates
(e.g. EGU, and those motivated by the Burg and Inverse divergences)
can be reparameterized as GD updates.
Concretely, our contributions can be summarized as follows.
\vspace{-0.2cm}
\begin{itemize}[leftmargin=3mm]
  \item We cast continuous MD updates as
minimizing a trade off between a \emph{Bregman momentum} and the loss.
We also derive the dual, natural gradient, and the constraint
	versions of the updates.
  \item We then provide a general framework that allows reparameterizing one CMD update by another. It requires the existence of a certain reparameterization function and a condition on the derivatives of the two link functions as well as the reparameterization function.
  \item Specifically, we show that on certain problems, the
      implicit bias of the GD updates can be controlled by
	considering a family of {\it tempered} updates
	(parameterized by a \emph{temperature} $\tau \in
	\RR$) that interpolate between GD (with $\tau=0$)
	and EGU (with $\tau=1$), while covering a wider class of updates.
\end{itemize}
We conclude the paper with 
a number of open problems for future research directions.

\paragraph{Previous work}
There has been an increasing amount of of interest recently in determining the implicit bias of learning algorithms~\citep{srebro1,cnn-bias,optimal}. Here, we mainly focus on the MD updates.
The special case of reparameterizing continuous EGU as
continuous GD was already known~\citep{Akin79,wincolt}.
In this paper, we develop a more general framework for reparameterizing one CMD update by another. We give a large variety of examples
for reparameterizing the CMD updates as continuous GD
updates. The main new examples we consider are based on the tempered versions of the relative entropy divergence~\citep{bitemp}.
The main open problem regarding the CMD updates is whether the discretization of the reparameterized updates track the discretization of
the original (discretized) MD updates.
The strongest methodology for showing this would be to prove the
same regret bounds for the discretized reparameterized
update as for the original. This has been done in a case-by-case basis for the EG family~\citep{wincolt}. For more
discussion see the conclusion section, where we also discuss how our reparameterization method allows exploring the effect of the structure of the network on the implicit bias.

\paragraph{Some basic notation}
We use $\odot$, $\oslash$, and superscript$\,^\odot$ for
element-wise product, division, and power, respectively.
We let $\o(t)$ denote the weight or parameter vector as a function of time $t$.
Learning proceeds in steps. During step $s$, we start with
weight vector $\o(s)=\o_s$ and go to $\o(s+1)=\o_{s+1}$
while processing a batch of examples.
We also write the Jacobian of vector valued function
$q$ as $\J_{\!q}$\, and use $\H_F$ to denote the Hessian of
a scalar function $F$.
Furthermore, we let $\nabla_\w F(\w(t))$ denote the gradient of function $F(\w)$
evaluated at $\w(t)$ and often drop the subscript $\w$.

\section{Continuous-time Mirror Descent}
For a strictly convex, continuously-differentiable function $F: \dom \rightarrow \RR$ with convex domain $\dom\subseteq \RR^d $,
the \emph{Bregman divergence} between $\ot, \o \in \dom$ is defined as
    \[
        D_F(\ot, \o) \coloneqq F(\ot)\! -\! F(\o) \!-\!
	f(\o)^\top (\ot \!-\! \o)\, ,
\]
where $f \coloneqq \nabla F(\o)$
denotes the gradient of $F$, sometimes called the {\it link function}%
\footnote{%
The gradient of a scalar function is a special case of a Jacobian,
and should therefore be denoted by a row vector.
However, in this paper we use the more common column vector notation for gradients, i.e.  $\nabla F(\o) \coloneqq (\frac{\partial F}{\partial \o})^\top$.}.
Trading off the divergence to the last parameter $\o_s$
with the current loss lets us motivate the iterative {\it mirror
descent} (MD) updates \citep{mirror,eg}:
\begin{align}
        \label{eq:bregman-plus-loss}
	\o_{s+1} &= \argmin_{\o}\, \sfrac{1}{\eta}\,D_F(\o, \o_s) + L(\o)\, ,
\intertext{%
where $\eta > 0$ is often called the \emph{learning rate}.  Solving for $\o_{s+1}$ yields the so-called
{\it prox} or {\it implicit update}~\citep{rockafellar}:
	}
\label{eq:md-implicit}
	f(\o_{s+1}) &=f(\o_s)- \eta\, \nabla L(\o_{s+1})\, .
\intertext{%
This update is typically approximated
by the following {\it explicit} update that uses the gradient at the old
parameter $\o_s$ instead (denoted hear as the MD update):
}
    f(\o_{s+1}) &=f(\o_s)- \eta\, \nabla L(\o_s)\, .
    \hspace{20mm} \text{(MD)}
\label{eq:md-explicit}
\end{align}
We now show that the CMD update~\eqref{eq:cont-md}
can be motivated similarly by replacing
the Bregman divergence in the minimization problem~\eqref{eq:bregman-plus-loss}
with a ``momentum'' version which quantifies the rate of change in the
value of Bregman divergence as $\o(t)$ varies over time.
For the convex function $F$, we define the {\it Bregman momentum} between $\o(t), \o_0 \in \dom$ as the time differential of the Bregman divergence induced by $F$,
 \begin{align*}
     \dot{D}_F(\o(t), \o_0) & = \dot{F}(\o(t)) - f(\o_0)^\top \dot{\o}(t)
     = \big(f(\o(t)) - f(\o_0)\big)^\top \dot{\o}(t)\, .
    \end{align*}
\begin{theorem}
    The CMD update\footnote{%
	An equivalent integral form of the CMD update is
$\;\w(t)=f^{-1}\Big(f(\w_s)-\eta \int_{z=s}^t \nabla L(\w(z))\, d z\Big).$}
$$
       \dot{f}\big(\o(t)\big)=-\eta\,\nabla L(\o(t))\, ,
    \text{ with $\o(s) = \o_s,$}
$$
    is the solution of the following functional:
\begin{equation}
  \label{eq:main-obj}
    \min_{\text{curve }\o(t)}\;\Big\{ \sfrac{1}{\eta}\,\dot{D}_F(\o(t), \o_s) + L(\o(t))\Big\}\, .
  \end{equation}
\end{theorem}
\begin{proof}
    Setting the derivatives w.r.t. $\o(t)$ to zero, we have
     \begin{align*}
	& \frac{\partial}{\partial
	\o(t)}\Big(\big(f(\o(t)) - f(\o_s)\big)^\top \dot{\o}(t) + \eta\,
	L(\o(t))\Big)
	 \\ &\quad
	 =\H_F(\o(t))\,\dot{\o}(t) + \frac{\partial \dot{\o}(t)}{\partial \o(t)}
\big(f(\o(t)) - f(\o_s)\big) + \eta\,\nabla L(\o(t))\\
         &\quad = \dot{f}\big(\o(t)\big) + \eta\, \nabla L(\o(t)) = \zero\, ,
        \end{align*}
        where we use the fact
	that $\o(t)\,\text{and}\,\dot{\o}(t)$ are independent
	variables~\citep{burke} \& thus$\,\frac{\partial
	\dot{\o}(t)}{\partial \o(t)}\!\! =\!
	\zero$.\!\!\!\!
\end{proof}
Note that the implicit update~\eqref{eq:md-implicit} and the explicit update~\eqref{eq:md-explicit} can both be realized as the backward and the forward Euler approximations of~\eqref{eq:cont-md}, respectively. Alternatively, \eqref{eq:bregman-plus-loss} can be obtained from~\eqref{eq:main-obj} via a simple discretization of the momentum term (see Appendix~\ref{app:discr}).

We can provide an alternative definition of Bregman momentum
in terms of the dual of $F$ function. If $F^*(\w^*) =
\sup_{\otil \in \dom} \big(\otil^\top \w^*- F(\otil)\big)$
denotes the Fenchel dual of $F$ and $\o = \argsup_{\otil
\in \dom} (\otil^\top \w^* - F(\otil))$,
then the following relation holds between the pair of dual
variables $(\o, \w^*)$:
\begin{equation}
    \label{eq:primal-dual}
    \o = f^*(\w^*)\, , \quad \w^* = f(\o)\, , \quad \text{ and }\,\,\, f^* = f^{-1}\, .
\end{equation}
Taking the derivative of $\o(t)$ and $\w^*(t)$ w.r.t. $t$ yields:\\
\begin{minipage}[t]{.5\textwidth}
\vspace{-5mm}
\begin{align}
    \label{eq:wdots}
        \dot{\o}(t) & = \dot{f^*}\big(\w^*(t)\big) =
	\H_{F^*}\big(\w^*(t)\big)\wstardot\, ,
\end{align}
\end{minipage}%
\begin{minipage}[t]{.5\textwidth}
\vspace{-5mm}
\begin{align}
    \label{eq:tdots}
        \wstardot & = \dot{f}\big(\o(t)\big) = \H_F\big(\o(t)\big)\, \dot{\o}(t)\, .
\end{align}
\end{minipage}\\[1mm]
This pairing allows rewriting the Bregman momentum in its dual form:
    \begin{equation}
        \label{eq:dual-mom}
        \dot{D}_F(\o(t), \o_0) =
	\dot{D}_{F^*}(\w_0^*,\w^*(t))\\
        = (\w^*(t) - \w_0^*)^\top \H_{F^*}(\w^*(t))\,
	\wstardot\, .
    \end{equation}
An expanded derivation is given in Appendix~\ref{a:dualbreg}.
Using \eqref{eq:tdots}, we can rewrite the CMD update~\eqref{eq:cont-md} as
\begin{equation}
    \label{eq:w-natural}
    \dot{\o}(t) = -\eta\, \H_{F}^{-1}(\o(t))\,\nabla L(\o(t))\, , \hspace{1cm} \text{(NGD)}
\end{equation}
i.e. a natural gradient descent (NGD) update~\citep{natgrad} w.r.t. the Riemannian metric $\H_F$.
Using $\nabla L(\o) = \H_{F^*}(\w^*)\nabla_{\w^*} L\!\circ\!
f^*(\w^*)$ and $\H_F(\o) = \H_{F^*}^{-1}(\w^*)$,
the CMD update~\eqref{eq:cont-md}
can be written equivalently in the dual domain $\w^*$ as an NGD update w.r.t. the Riemannian
metric $\H_{F^*}$, or by applying \eqref{eq:wdots} as a CMD with the link $f^*$:\\
\begin{minipage}[t]{.55\textwidth}
\vspace{-5mm}
\begin{align}
    \wstardot &= -\eta\, \H_{F^*}^{-1}(\w^*(t))\,\nabla_{\w^*} L\!\circ\! f^*(\w^*(t))\, ,\label{eq:theta-natural}
\end{align}
\end{minipage}%
\begin{minipage}[t]{.45\textwidth}
\vspace{-5.5mm}
\begin{align}
    \dot{f^*}(\w^*(t)) & = -\eta\, \nabla_{\w^*} L\!\circ\!
    f^*(\w^*(t))\, .\label{eq:dcont-md}
\end{align}
\end{minipage}\\[1mm]
The equivalence of the primal-dual updates was
already shown in~\citep{jagota} for the continuous case and
in~\citep{raskutti} for the discrete case (where it only
holds in one direction).
We will show that the equivalence relation is a special case of the reparameterization theorem, introduced in the next section. In the following, we discuss the projected CMD updates for the constrained setting.
\begin{proposition}
  \label{prop:proj-cmd}
    The CMD update with the additional constraint $\psi\big(\o(t)\big) = \zero$\, for some function $\psi:\, \RR^d \rightarrow \RR^m$ s.t. $\{\o\in \C|\,\psi\big(\o(t)\big) = \zero\}$ is non-empty, amounts to the projected  gradient update
    \begin{equation}
        \label{eq:proj-cont-md}
        \dot{f}\big(\o(t)\big)=-\eta\,\P_\psi(\o(t))\nabla L(\o(t))
	\;\&\;
        \dot{f^*}(\w^*(t))=-\eta\, \P_\psi(\o(t))^\top\,\nabla L \Circ f^*\, (\w^*(t))\,,
    \end{equation}
    where $\P_\psi \coloneqq \I_d -  \J^\top_\psi\big(\J_\psi\H^{-1}_F\J^\top_\psi\big)^{-1}
     \J_\psi\H^{-1}_F$
    is the projection matrix onto the tangent space of $F\;$
    at $\o(t)$ and $\J_\psi(\o(t))$.
    Equivalently, the update can be written as a projected natural gradient descent update
    \begin{equation}
      \label{eq:proj-cont-ngd}
      \dot{\o}(t) \!=\! -\eta  \P^\top_\psi(\o(t)) \H_F^{-1}(\o(t)) \nabla L(\o(t))
	\,\&\,
	\wstardot\!=\!-\eta \P_\psi\H_{F^*}^{-1}(\w^*(t)) \nabla L\Circ
	f^*(\w^*(t)).\!\!\!
    \end{equation}
\end{proposition}
\begin{example}[(Normalized) EG]
  The unnormalized EG update is motivated using the link function $f(\w) = \log\w$.
  Adding the linear constraint
  $\psi(\w)= \w^\top \one-1$ to the unnormalized EG update
    results in the (normalized) EG update~\citep{eg}.
  Since $\J_\psi(\w)=\one^\top$ and $\H_F(\w)^{-1}=\diag(\w)$,
  $\P_\psi= \I - \frac{\one \one^\top\diag(\w)}
               {\one^\top\!\! \diag(\w) \one}=\I -  \one\w^\top$
  and the projected CMD update \eqref{eq:proj-cont-ngd}
  (the continuous EG update) and its NGD form become
      \begin{align*}
  	\dot{\log}(\o)&=- \eta\; (\I -  \one\w^\top)\;\nabla L(\w)
  	= -\eta\; (\nabla L(\w) - \one \,\w^\top \nabla L(\w))\, ,
  	\\
  	\dot{\o}&= -\eta\; (\diag(\w)\nabla L(\w)-\w \,\w^\top \nabla L(\w))\, .
      \end{align*}
\end{example}

\section{Reparameterization}
\label{sec:repam}
We now establish the main result of the paper.
\begin{theorem} \label{thm:repam}
    Let $F$  and $G$ be strictly convex, continuously-differentiable functions with domains in $\RR^d$ and $\RR^k$, respectively, s.t. $k\ge d$.
    Let $q:\RR^k \rightarrow\RR^d$ be a \emph{reparameterization function}
    expressing parameters $\o$ of $F$ uniquely as $q(\u)$ where $\u$
    lies in the domain of $G$.
    Then the CMD update on parameter $\o$ for the convex function $F$
    (with link $f(\o)=\nabla F(\o)$) and loss $L(\w)$,
    $$\dot{f}(\o(t))=-\eta\,\nabla L(\o(t))\,,$$
    coincides with the CMD update on parameters $\u$
    for the convex function $G$ (with link $g(\u) \coloneqq \nabla G(\u)$)
    and the composite loss $L\!\circ\!q$,
    \[\dot{g}(\u(t)) =
    -\nabla_{\!\u}L\!\circ\!q\big(\u(t)\big)\, ,\] provided
    that $\range(q) \subseteq \domain(F)$ holds and we have
$$      \H_F^{-1}(\o) = \J_{\!q}(\u)\, \H_G^{-1}(\u)\,
    \J_{\!q}(\u)^\top,\text{ for all $\o = q(\u)$\,.}
$$
\end{theorem}
\begin{proof}
    Note that (dropping $t$ for simplicity) we have $\dot{\o} = \frac{\partial \o}{\partial\u}\,\dot{\u} = \J_{\! q}(\u)\, \dot{\u}$
    and $\nabla_{\!\u} L\!\circ\! q(\u) = \J_{\!q}(\u)^\top\nabla L(\o)$. The CMD update on $\u$ with the link function $g(\u)$ can be written in the NGD form as 
    $\dot{\u} = -\eta\, \H_G^{-1}(\u)\nabla_{\!\u}L\!\circ\! q(\u)$
    . Thus,
    \[
        \dot{\u} = -\eta\, \H_G^{-1}(\u)\,\J_{\!q}(\u)^\top\, \nabla_{\o} L(\o)\, .
    \]
    Multiplying by $\J_{\!q}(\u)$ from the left yields
    \[
         \dot{\o} = - \eta\, \J_{\!q}(\u) \H_G^{-1}(\u) \J_{\!q}(\u)^\top\nabla_{\o} L(\o)\, .
    \]
    Comparing the result to~\eqref{eq:w-natural} concludes the proof.
\end{proof}
In the following examples, we will mainly consider reparameterizing a CMD update with the link function $f(\w)$ as a GD update on $\u$, for which we have $\H_{G} = \I_k$.
\begin{example}[EGU as GD]
  \label{ex:egu-gd}
  The continuous-time EGU can be reparameterized as
    continuous GD with the reparameterization function
	$\o =q(\u)= \sfrac14\,\u\odot\u = \sfrac14\,\u^{\odot 2}$, i.e.
$$	\dot{\log}(\o) = -\eta\; \nabla L(\o)
	\;\; \text{equals} \;\;
	\dot{\u} = -\eta \underbrace{\nabla L\Circ q\,(\u)}_
			 {\nabla_{\u} L\,(\sfrac14\, \u^{\odot 2})}
			 = {-\sfrac{\eta}{2}\, \u \odot \nabla L(\o)}
$$
    This is proven by verifying the condition of Theorem~\ref{thm:repam}:
   \begin{align*}
\J_{\!q}(\u)\J_{\!q}(\u)^\top & = \half \diag(\u) \,(\half\diag(\u))^\top=\diag(\sfrac14\, \u^{\odot 2})=\diag(\o) = \H^{-1}_F(\o)\, .
     \end{align*}
\end{example}
\begin{example}[Reduced EG in $2$-dimension]
Consider the $2$-dimensional normalized weights
$\o = [\,\omega, 1-\omega]^\top$ where $0 \leq \omega \leq 1$.
The normalized reduced EG update~\citep{jagota} is
motivated by the link function $f(w) = \log\frac{w}{1 -
w}$, thus $H_F(w) = \frac{1}{w} + \frac{1}{1 - w} = \frac{1}{w(1-w)}$.
This update can be reparameterized as a GD update on $u \in \RR$ via $\omega = q(u) =
\sfrac{1}{2}(1 + \sin(u))$ i.e.
\[
\dot{\log}(\frac{w}{1 - w}) = -\eta\;\nabla_w L(w)
\;\; \text{equals} \;\;
\dot{u} = -\eta\!\! \underbrace{\nabla_u L\!\circ\! q\,(u)}_
     {\nabla_u L\,\big(\sfrac{1}{2}(1 + \sin(u))\big)}
     = {-\eta\, \frac{\cos(u)}{2}\nabla L(w)}\,.
     \]
     This is verified by checking the condition of Theorem~\ref{thm:repam}:
  $J_q(u) = \sfrac{1}{2}\, \cos(u)$ and
\[
J_q(u) J_q(u)^\top = \frac{1}{4}\,\cos^2(u) = \frac{1}{2}\,\big(1+\sin(u)\big)\, \frac{1}{2}\big(1-\sin(u)\big) = w(1-w) = H_F^{-1}(w)\, .
\]
\end{example}
\textbf{Open problem}\, The generalization of the reduced
EG link function to $d>2$ dimensions becomes $f(\w)=\log\frac{\w}{1-\sum_{i=1}^{d-1} w_i}$
which utilizes the first $(d-1)$-dimensions $\w$ s.t.
$[\w^\top, w_d]^\top \in \simplex^{d-1}$.
Reparameterizing the CMD update using this link as CGD is open.
The update can be reformulated as
\begin{align*}
	\dot{\o}
	=-\eta\; \Big(
        \mathrm{diag}\big(\frac1{\o}\big)+\frac{1}{1-\sum_{i=1}^{d-1}w_i}\,\bm{1}\bm{1}^\top \Big)^{-1} \nabla L(\w)
    =-\eta \left( \mathrm{diag}(\o)-\o\o^\top \right) \nabla\, L(\w)\, .
    \end{align*}

Later, we will give an $d$-dimensional version of EG using a
projection onto a constraint.
\begin{example}[Burg updates as GD]
  \label{ex:burg}
    The update associated with the negative Burg entropy
    $F(\o)=-\sum_{i=1}^d \log w_i$ and link $f(\o)=-\one\oslash\o$
    is reparameterized as GD with $\o=q(\u):=\exp(\u)$, i.e.
  	\[
    \dot{(-\one\oslash\o)} = -\eta\;  \nabla L(\o)
  	\;\text{equals} \;
  	\dot{\u} = -\eta\; \underbrace{\nabla L\Circ q\,(\u)}_{\nabla_{\u} L\,(\exp(\u))}
  	= -\eta \,\exp(\u) \!\odot\! \nabla L(\o)\, ,
   \]
This is verified by the condition of
    Theorem~\ref{thm:repam}: $\H_{F}(\w)
    = \diag(\one\oslash\w)^2$, $\J_{\!q}(\u) = \diag(\exp(\u))$, and
\[
\J_{\!q}(\u)\J_{\!q}(\u)^\top = \diag(\exp(\u))^2 = \diag(\w)^2 = \H^{-1}_{F}(\w)\, .
\]
\end{example}
\begin{example}[EGU as Burg]
  The reparameterization step can be chained, and applied
    in reverse, when the reparameterization function $q$ is invertible.
    For instance, we can first apply the inverse
    reparameterization of the Burg update as GD
    from Example~\ref{ex:burg}, i.e. $\u = q^{-1}(\w) = \log \w$.
    Subsequently, applying the reparameterization of EGU as
    GD from Example~\ref{ex:egu-gd}, i.e. $\v =
    \tilde{q}(\u) = \sfrac14\,\u^{\odot 2}$, results in the reparameterization of EGU as Burg update, that is,
  $$\dot{\log}(\v) = -\eta\; \nabla L(\v)
  	\;\;
  \;\text{equals} \;
  \dot{\left(-\frac1{\o}\right)} = -\eta\;
    \underbrace{\nabla_{\w}L\Circ \tilde{q}\Circ q^{-1}(\o)}
    _{\nabla_{\w}L(\sfrac14 (\log\w)^{\odot 2})}
    = -\eta (\log(\w)\oslash (2\w))\odot \nabla L(\v)
    \, .$$
\end{example}
For completeness,
we also provide the constrained reparameterized updates
(proof in Appendix~\ref{app:const-reparam}).
\begin{theorem}
    \label{thm:proj-reparam}
     The constrained CMD update~\eqref{eq:proj-cont-md} coincides with the reparameterized projected gradient update on the composite loss,
     \[
         \dot{g}\big(\u(t)\big)=-\eta\,\P_{\psi\circ q}(\u(t))\nabla_{\u} L\circ q(\u(t))\, ,
    \]
    where
    $\P_{\psi\circ q} \coloneqq \I_k -   \J^\top_{\psi\circ
    q}\big(\J_{\psi\circ q}\H^{-1}_G\J^\top_{\psi\circ
    q}\big)^{-1} \J_{\psi\circ q}\H^{-1}_G$
 is the projection matrix onto the tangent space at $\u(t)$
    and $\J_{\psi\circ q}(\u) \coloneqq
    \J_q^\top(\u) \J_\psi (\o)$.
\end{theorem}
\begin{example}[EG as GD]
We now extend the reparameterization of the EGU update as GD in Example~\ref{ex:egu-gd} to the normalized case in terms of a projected GD update. Combining $q(\u) = \sfrac{1}{4}\, \u^{\odot 2}$ with $\psi(\w) = \one^\top\w - 1$, we have $\J_{\psi\circ q}(\u) = \sfrac12\diag(\u)\,\one^\top = \u^\top$ and $\P_{\psi\circ q}(\u) = \I - \frac{\u\u^\top}{\Vert\u\Vert^2}$. Thus,
\begin{align*}
\dot{\u}&=- \eta \big(\I - \frac{\u\u^\top}{\Vert\u\Vert^2}\big)\nabla_{\!\u} L\!\circ\! q(\u) = -\sfrac{\eta}{2}\big(\u - \u\u^\top\big)\,\nabla_{\u}L(\sfrac14\,\u^{\odot 2}) \text{ with } \w(t) = \sfrac14\,\u(t)^{\odot 2}
\end{align*}
equals the normalized EG update in Example~\ref{ex:egu-gd}. Note that similar ideas was explored in an evolutionary game theory context in~\citep{sandholm}.
\end{example}

\section{Tempered Updates}

\begin{wrapfigure}{r}{0.3\textwidth}
    \vspace{-1.5cm}
    \begin{center}
    \includegraphics[width=0.32\textwidth]{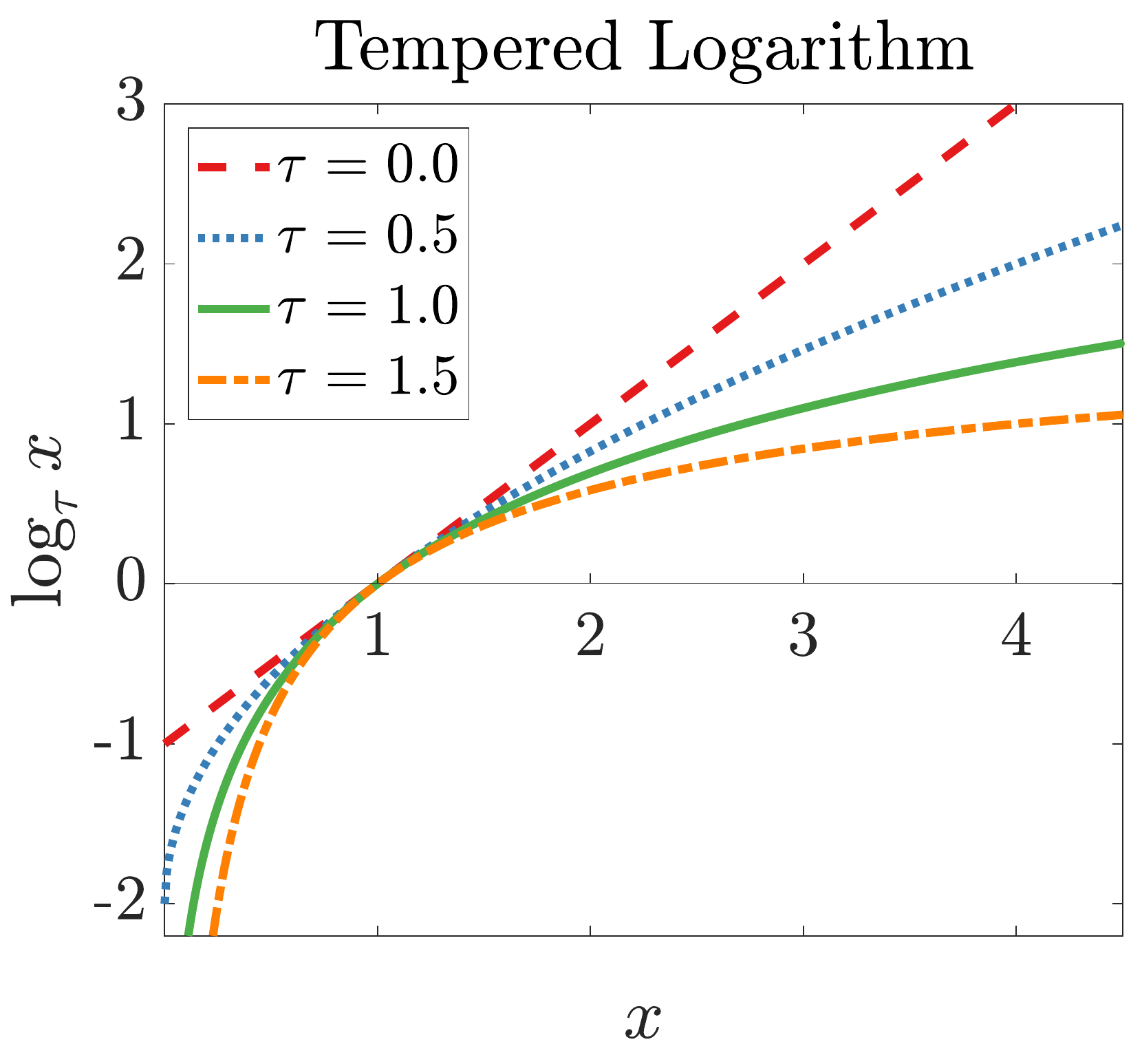}
    \vspace{-6mm}
    \caption{$\log_{\tau}(x)$, for different $\tau\geq 0$.}
    \label{fig:log}
    \end{center}
    \vspace{-3mm}
\end{wrapfigure}
In this section, we consider a richer class of examples derived using the tempered relative entropy divergence~\citep{bitemp}, parameterized by a \emph{temperature} $\tau \in \RR$. As we will see, the tempered updates allow interpolating between many well-known cases. We start with the tempered logarithm link function~\citep{texp1}:
\begin{equation}
  \label{eq:logt}
    f_\tau(\o) = \log_\tau(\o) = \frac{1}{1-\tau} (\o^{1-\tau} -1)\, ,
\end{equation}
for $\o \in \RR^d_{\geq 0}$ and $\tau\in\RR$.
The $\log_\tau$ function is shown in Figure~\ref{fig:log}
for different values of $\tau\geq 0$.
Note that $\tau = 1$ recovers the standard $\log$ function
as a limit point. The $\log_\tau(\o)$ link function is the gradient of the convex function
\begin{align*}
    & F_\tau(\o) = \sum_i \big(w_i \log_\tau w_i + \frac{1}{2\!-\!\tau}\, (1-w_i^{2-\tau})\big)
    = \sum_i \Big(\frac{1}{(1\!-\!\tau)(2\!-\!\tau)}\, w_i^{2\!-\!\tau} - \frac{1}{1\!-\!\tau}\, w_i + \frac{1}{2\!-\!\tau}\Big)\, .
\end{align*}
The convex function $F_\tau$ induces the following tempered Bregman divergence\footnote{The second form is more commonly known as $\beta$-divergence~\citep{beta} with $\beta = 2 - \tau$.}:
\begin{align}
    \label{eq:temp-div}
    D_{F_\tau}(\ot, \o)
    = &\sum_i\! \Big(\ott_i \log_\tau \ott_i -
    \ott_i \log_\tau w_i - \frac{\ott_i^{2-\tau} -
    w_i^{2-\tau}}{2-\tau}\!\Big)
    \nonumber\\
    =& \frac{1}{1-\tau}
    \sum_i\! \Big( \frac{\ott_i^{2-\tau}\!-w_i^{2-\tau}}{2-\tau}
    -(\ott_i-w_i)\,w_i^{1-\tau}\!\Big) .
\end{align}
For $\tau=0$, we obtain the squared Euclidean divergence
$D_{F_0}(\ot, \o) = \frac{1}{2}\,\Vert\ot -
\o\Vert_2^2$ and for $\tau=1$, the relative entropy
$D_{F_1}(\ot, \o) = \sum_i (\ott_i \log
(\sfrac{\ott_i}{w_i}) - \ott_i + w_i)$
(See \citep{bitemp} for an extensive list of examples).

In the following, we derive the CMD updates using the time derivative of~\eqref{eq:temp-div} as the tempered Bregman momentum. Notice that the link function $\log_{\tau}(x)$ is only defined for $x \geq 0$ when $\tau > 0$. In order to have a weight $\o \in \RR^d$, we use the $\pm$-trick~\citep{eg} by maintaining two non-negative weights $\o_+$ and $\o_-$ and setting $\o = \o_+ - \o_-$. We call this the \emph{tempered EGU$^\pm$} updates, which contain the standard EGU$^\pm$ updates as a special case of $\tau = 1$. As our second main result, we show that
that continuous tempered EGU$^\pm$ updates interpolate between
continuous-time GD and continuous EGU (for $\tau\in [0,1]$).
Furthermore, these updates can be simulated by continuous GD on a new set of parameters $\u$ using a simple reparameterization. We show that reparameterizing the tempered updates as GD updates on the composite loss $L\Circ q$ changes the implicit bias of the GD, making the updates to converge to the solution with the smallest $\mathrm{L}_{2-\tau}$-norm for arbitrary $\tau \in [0, 1]$.

\subsection{Tempered EGU and Reparameterization}
We first introduce the generalization of the EGU update using the tempered Bregman
divergence~\eqref{eq:temp-div}. Let $\o(t)\in \RR^d_{\geq 0}$. The tempered EGU update is motivated by
\begin{align*}
    &\argmin_{\text{curve }\o(t)\in \RR^d_{\geq 0}}
    \Big\{\sfrac{1}{\eta}\, \dot{D}_{F_\tau}\big(\o(t), \o_0\big) + L(\o(t))\Big\}\,.
\end{align*}
This results in the CMD update
\begin{equation}
  \label{eq:tempered-egu}
  \dot{\log}_\tau \o(t) = - \nabla L(\o(t))\, .
\end{equation}
\begin{wrapfigure}{r}{0.37\textwidth}
    \vspace{-2mm}
\begin{center}
\includegraphics[width=0.35\textwidth]{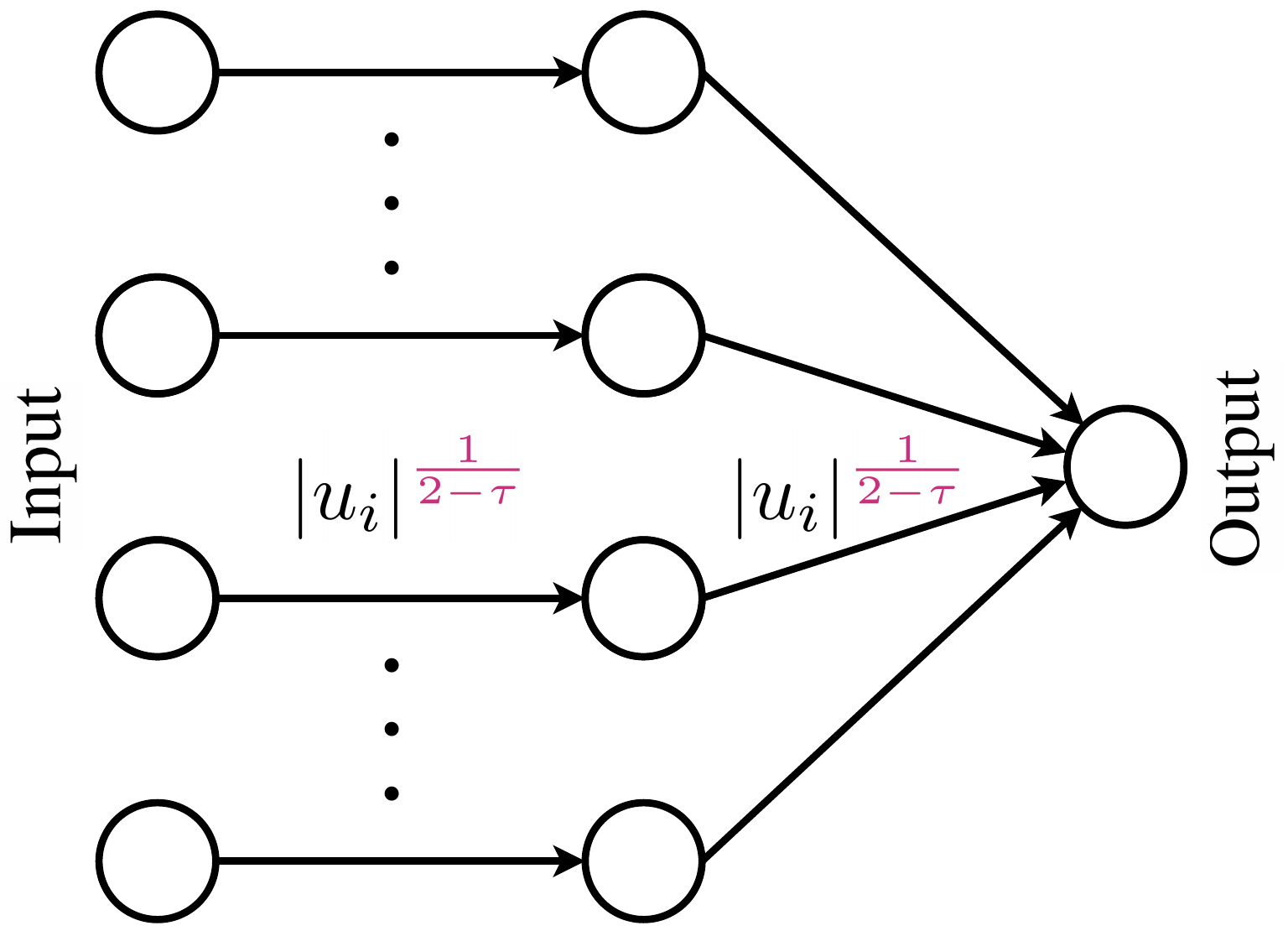}
\vspace{-3mm}
\caption{A reparameterized linear neuron where $w_i = \vert u_i\vert^{\frac{2}{2-\tau}}$ as a two-layer sparse network: value of $\tau=0$ reduces to GD while $\tau=1$ simulates the EGU update.}
\label{fig:neuron}
\end{center}
\vspace{-5mm}
\end{wrapfigure}
An equivalent integral version of this update is
\begin{equation}
\label{eq:temp-egu-int}
\o(t) = \exp_\tau\!\big(\log_\tau \o_0
- \eta\,\int_0^t\!\! \nabla_{\o} L(\o(z))\, \dd z\big),
\end{equation}
where $\exp_\tau(x) \coloneqq [1 + (1-\tau) x]_+^{\frac{1}{1-\tau}}$ is the inverse of tempered logarithm~\eqref{eq:logt}.
    Note that $\tau = 1$ is a limit case which recovers the standard $\exp$ function
    and the update~\eqref{eq:tempered-egu} becomes the
    standard EGU update. Additionally, the GD update (on
    the non-negative orthant) is recovered at $\tau = 0$.
    As a result, the tempered EGU
    update~\eqref{eq:tempered-egu} interpolates between GD
    and EGU for $\tau\in [0,1]$ and generalize beyond for
    values of $\tau > 1$ and $\tau < 0$.\footnote{%
	For example, $\tau=2$ corresponds to the Burg
	updates (Example~\ref{ex:burg}).}
    We now show the
    reparameterization of the tempered EGU
    update~\eqref{eq:tempered-egu} as GD.
    This corresponds
    to continuous-time gradient descent on the network of Figure
    \ref{fig:neuron}.
    \begin{proposition}
        \label{prop:reparam}
    The tempered continuous EGU update can be
	reparameterized continuous-time GD with the reparameterization function
        \begin{equation}
            \label{eq:q-tau}
            \w= q_{\tau}(\u) =
	    \big(\frac{2-\tau}2\big)^{\frac2{2-\tau}}
	    \vert\u\vert^{\odot \frac2{2-\tau}}\, ,\,\, \text{for
	    }\u\in \RR^d \text{ and } \tau \neq 2\, .
        \end{equation}
	That is
	\begin{align*}
	\dot{\log_\tau}(\w) = -\eta  \nabla &\! L(\w)
	\; \text{equals} \;
	\dot{\u} = -\eta
	    \!\!\!\!\!\!\!\!\!\!\!\!\!\!\!
	    \underbrace{\nabla L\Circ q_\tau\,(\u)}_
	{\nabla_{\u} L\big(\big(\frac{2-\tau}2\big)^{\frac{2}{2-\tau}}
				  |\u|^{{\odot \frac2{2-\tau}}}\big)}
            \!\!\!\!\!\!\!\!\!\!\!\!\!\!\!
	    = {-\eta \sign(\u)\!\odot\!
	    \Big(\frac{2-\tau}2\Big)^{\frac{\tau}{2-\tau}}
	    \!|\u|^{\odot\frac{\tau}{2-\tau}} \!\odot\! \nabla L(\o)}.
     \end{align*}
    \end{proposition}
    \begin{proof}
     This is verified by checking the condition of
	Theorem~\ref{thm:repam}. The lhs is
	\begin{align*}
	    { (\H_{F_\tau(\o)}(\w))^{-1}=(\J_{\log_{\tau}}(\o))^{-1} }&
	{=(\diag(\o)^{-\tau})^{-1}=\diag(\o)^\tau}.
	\end{align*}
	Note that the Jacobian of $q_{\tau}$ is
    \[\J_{\!q_\tau}(\u) =
	\Big(\frac{2-\tau}2\Big)^{\frac{\tau}{2-\tau}}\diag\big(\sign(\u)\odot\vert\u\vert^{\odot\frac{\tau}{2-\tau}}\big)
        = \diag(\sign(\u)\odot q_{\tau}(\u)^{\odot\frac{\tau}2}).\]
    Thus the rhs
        $\J_{\!q_\tau}(\u)\J_{\!q_\tau}^\top(\u) $
	of the condition equals
	$\diag\big(\o^{\odot\tau}\big)$
	as well.
    \end{proof}
\subsection{Minimum-norm Solutions}
We apply the (reparameterized) tempered EGU update on the under-determined linear regression problem. For this, we first consider the $\pm$-trick on~\eqref{eq:tempered-egu}, in which we set $\w(t) = \w_+(t) - \w_-(t)$ where
\begin{equation}
    \label{eq:temp-dyn}
        \dot{\log_\tau} \o_+(t) = -\eta\,
	\nabla_{\o}L(\o(t))\,, \quad
        \dot{\log_\tau} \o_-(t) = +\eta\,  \nabla_{\o}L(\o(t))\, .
\end{equation}
Note that using the $\pm$-trick, we have $\w(t) \in \RR^n$. We call the updates~\eqref{eq:temp-dyn} the \emph{tempered EGU$^\pm$}. The reparameterization of the tempered EGU$^\pm$ updates as GD can be written by applying Proposition~\ref{prop:reparam},
\begin{equation}
    \label{eq:temp-repam}
        \dot{\u}_+\!(t)  = -\eta\, \nabla_{\u_+} L\big(q_\tau(\u_+(t)) -  q_\tau(\u_-(t))\big)\,,
     \dot{\u}_-\!(t)  = -\eta\, \nabla_{\u_-} L\big(q_\tau(\u_+(t)) - q_\tau(\u_-(t))\big),
\end{equation}
and setting $\o(t) = q_\tau(\u_+(t)) - q_\tau(\u_-(t))$.

The strong convexity of the $F_\tau$ function w.r.t. the $\mathrm{L}_{2-\tau}$-norm (see~\citep{bitemp}) suggests that the updates motivated by the tempered Bregman divergence~\eqref{eq:temp-div} yield the
minimum $\mathrm{L}_{2-\tau}$-norm solution in certain settings. We verify this by considering the following under-determined linear regression problem. Let $\{\x_n, y_n\}_{n=1}^N$ where $\x_n\in \RR^d,\, y_n\in\RR$ denote the set of input-output pairs and let $\X\in\RR^{N\times d}$ be the design matrix for which the $n$-th row is equal to $\x_n^\top$. Also, let $\y\in\RR^N$ denote the vector of targets. Consider the tempered EGU$^\pm$ updates~\eqref{eq:temp-dyn} on the weights $\o(t) = \o_+(t) - \o_-(t)$ where $\o_+(t), \o_-(t) \geq \bm{0}$ and $\o_+(0) = \o_-(0) = \o_0$. Following~\eqref{eq:temp-egu-int}, we have
\[
\o_+(t) = \exp_\tau\big(\log_\tau \o_0 - \eta\, \int_0^t \X^\top\del(z) \, \dd z\big)\,,\,\,\o_-(t) = \exp_\tau\big(\log_\tau \o_0 + \eta\, \int_0^t \X^\top\del(z)\, \dd z\big)\,,
\]
where $\del(t) = \X\big(\o_+(t) - \o_-(t)\big)$\,.
\begin{theorem}
\label{prop:min-norm}
Consider the underdetermined linear regression problem where $ N< d$. Let $\mathcal{E} = \{\o\in \RR^d|\, \X\o = \y\}$ be the set of solutions with zero error. Given $\o(\infty) \in \mathcal{E}$, then the tempered EGU$^\pm$ updates~\eqref{eq:temp-dyn} with temperature $0 \leq \tau \leq 1$ and initial solution $\o_0 = \alpha\one \geq 0$ converge to the minimum $\mathrm{L}_{2-\tau}$-norm solution in $\mathcal{E}$ in the limit $\alpha \rightarrow 0$.
\end{theorem}
\begin{proof}
We show that the solution of the tempered EGU$^\pm$ satisfies the dual feasibility and complementary slackness KKT conditions for the following optimization problem (omitting $t$ for simplicity):
     \begin{align*}
          \quad \min_{\o_+, \o_-}\, \Vert \o_+ -
	 \o_-\Vert_{2-\tau}^{2-\tau}, \,\,\,\text{for }0 \leq \tau \leq 1,
 \quad \text{s.t.} \quad
 \X(\o_+-\o_-) = \y \,\,\,\text{ and }\,\,\, \o_+, \o_- \geq \bm{0}\, .
     \end{align*}
Imposing the constraints using a set of Lagrange multipliers $\bnu_+, \bnu_-\geq \bm{0}$ and $\lambda \in \RR$, we have
\begin{multline*}
\min_\o \sup_{\bnu_+, \bnu_-\geq \bm{0},
\lambda}\,\, \Big\{ \Vert \o_+ -
\o_-\Vert^{2-\tau}_{2-\tau} +
\lam^\top\,\big(\X(\o_+-\o_-) - \y\big) - \o_+^\top\bnu_+  - \o_-^\top \bnu_-\Big\}\,.
\end{multline*}
The set of KKT conditions are
\[
\begin{cases}
    \o_+, \o_- \geq \bm{0}\, ,\,\,\X\o = \y\, ,\\
    +\sign(\o)\odot \vert\o\vert^{\odot(1-\tau)} - \X^\top\lam \succcurlyeq \bm{0}\, ,\,\,
    -\sign(\o)\odot \vert\o\vert^{\odot(1-\tau)} + \X^\top\lam \succcurlyeq \bm{0}\, ,\\
    \big(\sign(\o)\odot \vert\o\vert^{\odot(1-\tau)}  - \X^\top\lam\big)\odot \o_+ = \bm{0}\, ,\,\,
\big(\sign(\o)\odot \vert\o\vert^{\odot(1-\tau)}  - \X^\top\lam\big)\odot
    \o_- = \bm{0}\,,
           \end{cases}
\]
where $\o = \o_+ - \o_-$. The first condition is imposed by the form of the updates and the second condition is satisfied by the assumption at $t\rightarrow \infty$. Using $\o_0 = \alpha\one$ with $\alpha \rightarrow 0$, we have
    \begin{align*}
        \o_+(t) & = \exp_\tau\big(-\frac{1}{1-\tau} - \eta\, \int_0^t \X^\top\del(z) \, \dd z\big) = \big[-(1-\tau)\,\eta\,\X \int_0^t \del(z)\big]_+^{\odot\frac{1}{1-\tau}}\, ,\\
        \o_-(t) & = \exp_\tau\big(-\frac{1}{1-\tau} + \eta\,
        \int_0^t \X^\top\del(z) \, \dd z\big) =
\big[+(1-\tau)\,\eta\,\X \int_0^t \del(z)\big]_+^{\odot\frac{1}{1-\tau}}\, .
    \end{align*}
Setting $\lam = - (1-\tau)\,\eta \int_0^\infty \del(z)$ satisfies the remaining  KKT conditions.
\end{proof}
\begin{corollary}
    \label{cor:bias}
Under the assumptions of Theorem~\ref{prop:min-norm}, the reparameterized tempered EGU$^\pm$ updates~\eqref{eq:temp-repam} also recover the minimum $\mathrm{L}_{2-\tau}$-norm solution where $\o(t) = q_{\tau}(\u_+(t)) - q_{\tau}(\u_-(t))$.
\end{corollary}
    This corollary shows that reparameterizing the loss in
    terms of the parameters $\u$ changes the implicit bias
    of the GD updates. Similar results were observed before
    in terms of sparse signal recovery~\citep{optimal} and
    matrix factorization~\citep{srebro1}. Here, we show
    that this is a direct result of the dynamics induced by
    the reparameterization Theorem~\ref{thm:repam}.

\section{Conclusion and Future Work}
In this paper, we discussed the continuous-time mirror
    descent updates and provided a general framework for
    reparameterizing these updates. Additionally, we
    introduced the tempered EGU$^\pm$ updates and their
    reparameterized forms. The tempered EGU$^\pm$ updates
    include the two commonly used gradient descent and exponentiated gradient updates,
    and interpolations between them.
For the underdetermined linear regression problem we showed that under certain conditions,
the tempered EGU$^\pm$ updates converge to the minimum $\mathrm{L}_{2-\tau}$-norm solution.
    The current work leads to many interesting future directions:
\begin{itemize}[leftmargin=3mm]
             \item
The focus is this paper was to develop the reparameterization method in full
	generality. Our reparameterization equivalence theorem holds only in the continuous-time
and the equivalence relation breaks down after discretization. However,
	in many important cases the discretized reparameterized updates closely track the
	discretized original updates~\citep{wincolt}.
	This was done by proving the same on-line worst
	case regret bounds for the discretized
	reparameterized updates and the originals.
	A key research direction is to find general conditions for which this is true.
                 \item Perhaps the most important application of the current work is reparameterizing the weights of deep neural networks for achieving sparse solutions or obtaining an implicit form of regularization that mimics a trade-off between the ridge and lasso methods (e.g. elastic net regularization~\citep{elastic}).
Here the deep open question is the following:
Are sparse networks (as in Figure \ref{fig:neuron}) required,
if the goal is to obtain sparse solutions efficiently?
             \item A more general treatment of the underdetermined linear regression
		 case requires analyzing the results for arbitrary start vectors.
		 Also, developing a matrix form of the reparameterization theorem is left for future work.
         \end{itemize}
\newpage
\section*{Broader Impact}
The result of the paper suggests that the mirror descent updates can be effectively used in neural networks by running backpropagation on the reparameterized form of the neurons. This may have a potential use case for training these networks more efficiently. This is a theoretical paper and the broader ethical impact discussion is not applicable.

\bibliographystyle{plainnat}
\bibliography{refs}

\newpage
\appendix

\section{Dual Form of Bregman Momentum}
\label{a:dualbreg}
The dual form of Bregman momentum given
in~\eqref{eq:dual-mom} can be obtained by first forming the
dual Bregman divergence in terms of the dual variables
$\w^*(t)$ and $\w_s^*$ and taking the time derivative, that is,
 \begin{align*}
     \dot{D}_F(\o(t), \o_0) &=
     \dot{D}_{F^*}(\w_0^*,\w^*(t))
      = \frac{\partial}{\partial t}\Big(F^*(\w_0^*) -
      F^*(\w^*(t)) - f^*(\w^*(t))^\top\big(\w_0^* - \w^*(t)\big)\\
     & =-\dot{F^*}(\w^*(t)) + f^*(\w^*(t))^\top \wstardot
        + (\w^*(t) - \w_0^*)^\top \H_{F^*}(\w^*(t))\,
	\wstardot\\
        & = \big(\w^*(t) - \w_0^*\big)^\top
	\H_{F^*}(\w^*(t))\, \wstardot\, ,
\end{align*}
where we use the fact
that $\dot{F^*}(\w^*(t)) =
f^*(\w^*(t))^\top \wstardot$.

\section{Constrained Updates and Reparameterization}
\label{app:const-reparam}
We first provide a proof for Proposition~\ref{prop:proj-cmd}. Then, we prove Theorem~\ref{thm:proj-reparam}.

\noindent \textbf{Proposition~\ref{prop:proj-cmd}.\, }\emph{The CMD update with the additional constraint $\psi\big(\o(t)\big) = \zero$\, for some function $\psi:\, \RR^d \rightarrow \RR^m$ s.t. $\{\o\in \C|\,\psi\big(\o(t)\big) = \zero\}$ is non-empty, amounts to the projected  gradient update
    \begin{equation}
        \dot{f}\big(\o(t)\big)=-\eta\,\P_\psi(\o(t))\nabla L(\o(t))
	\;\&\;
        \dot{f^*}(\w^*(t))=-\eta\, \P_\psi(\o(t))^\top\,\nabla L \Circ f^*\, (\w^*(t))\,,\tag{\ref{eq:proj-cont-md}}
    \end{equation}
    where $\P_\psi \coloneqq \I_d -  \J^\top_\psi\big(\J_\psi\H^{-1}_F\J^\top_\psi\big)^{-1}
     \J_\psi\H^{-1}_F$
    is the projection matrix onto the tangent space of $F\;$
    at $\o(t)$ and $\J_\psi(\o(t))$.
    Equivalently, the update can be written as a projected natural gradient descent update
    \begin{equation}
     \dot{\o}(t) \!=\! -\eta  \P^\top_\psi(\o(t)) \H_F^{-1}(\o(t)) \nabla L(\o(t))
	\,\&\,
	\wstardot\!=\!-\eta \P_\psi\H_{F^*}^{-1}(\w^*(t)) \nabla L\Circ
        f^*(\w^*(t)).\!\!\!\tag{\ref{eq:proj-cont-ngd}}
    \end{equation}}
\begin{proof}[Proof of Proposition~\ref{prop:proj-cmd}]
  We use a Lagrange multiplier $\lam(t) \in \RR^m$ in~\eqref{eq:main-obj} to enforce the constraint $\psi(\o(t)) = \zero$ for all $t \geq 0$,
  \begin{align}
    \label{eq:main-obj-lam}
\min_{\o(t)}\,\Big\{ \sfrac{1}{\eta}\,\dot{D}_F(\o(t), \o_s) + L(\o(t)) + \lam(t)^\top\psi(\o(t)) \Big\}\, .
\end{align}
Setting the derivative w.r.t. $\o(t)$ to zero, we have
\begin{equation}
  \label{eq:der-const-cmd}
\dot{f}(\o(t)) + \eta\, \nabla_{\o} L(\o(t)) +
    \J_\psi(\o(t))^\top \lam(t) = \zero\, ,
\end{equation}
    where $\J_\psi(\o(t))$ is the Jacobian of the function $\psi(\o(t))$. In order to solve for $\lam(t)$, first note that $\dot{\psi}(\o(t)) = \J_\psi(\o(t))\,\dot{\o}(t) = \zero$. Using the equality $\dot{f}(\o(t)) = \H_F(\o(t))\dot{\o}(t)$ and multiplying both sides by $\J_\psi(\o(t))\H^{-1}_F(\o(t))$ yields (ignoring $t$)
\begin{align*}
\cancel{\J_\psi(\o)\,\dot{\o}} + \eta\, \J_\psi(\o)\H^{-1}_F(\o)\nabla L(\o)
+ \J_\psi(\o)\H^{-1}_F(\o)\J^\top_\psi(\o)\lam(t) = \zero
\end{align*}
Assuming that the inverse exists, we can written
\begin{align*}
\lam = -\eta\, \big(\J_\psi(\o)\H^{-1}_F(\o)\J^\top_\psi(\o)\big)^{-1} \J_\psi(\o)\H^{-1}_F(\o)\nabla L(\o)\, .
\end{align*}
Plugging in for $\lam(t)$ yields~\eqref{eq:proj-cont-ngd}. Multiplying both sides by $\H_F(\o)$ and using $\dot{f}(\o) = \H_F(\o)\dot{\o}$ yields~\eqref{eq:proj-cont-md}.
\end{proof}

\noindent\textbf{Theorem~\ref{thm:proj-reparam}. }\emph{
     The constrained CMD update~\eqref{eq:proj-cont-md} coincides with the reparameterized projected gradient update on the composite loss,
     \[
         \dot{g}\big(\u(t)\big)=-\eta\,\P_{\psi\circ q}(\u(t))\nabla_{\u} L\circ q(\u(t))\, ,
    \]
    where
    $\P_{\psi\circ q} \coloneqq \I_k -   \J^\top_{\psi\circ
    q}\big(\J_{\psi\circ q}\H^{-1}_G\J^\top_{\psi\circ
    q}\big)^{-1} \J_{\psi\circ q}\H^{-1}_G$
 is the projection matrix onto the tangent space at $\u(t)$
    and $\J_{\psi\circ q}(\u) \coloneqq
    \J_q^\top(\u) \J_\psi (\o)$.
}
\begin{proof}[Proof of Theorem~\ref{thm:proj-reparam}]
    Similar to the proof of Proposition~\ref{prop:proj-cmd}, we use a Lagrange multiplier $\lam(t) \in \RR^m$ to enforce the constraint $\psi\circ q(\u(t)) = \zero$ for all $t \geq 0$,
  \begin{align*}
\min_{\u(t)}\,\Big\{ \sfrac{1}{\eta}\,\dot{D_G}(\u(t), \u_s) + L\!\circ\!q(\u(t)) + \lam(t)^\top\psi\Circ q(\u(t)) \Big\}\, .
\end{align*}
Setting the derivative w.r.t. $\u(t)$ to zero, we have
\[
    \dot{g}(\o(t)) + \eta\, \nabla_{\u} L\!\circ\!q(\o(t)) + \J^\top_{\psi\circ q}(\u(t)) \lam(t) = \zero\, ,
\]
    where $\J_{\psi\circ q}(\u(t)) \coloneqq \J^\top_q(\u)\nabla \psi(\o(t))$. In order to solve for $\lam(t)$, we use the fact that $\dot{\psi\circ q}(\u(t)) = \J_{\psi\circ q}(\u(t))\,\dot{\u}(t) = \zero$. Using the equality $\dot{g}(\u(t)) = \H_G(\u(t))\dot{\u}(t)$ and multiplying both sides by $\J_{\psi\circ q}(\u(t))\H^{-1}_G(\u(t))$ yields (ignoring $t$)
\begin{align*}
    \J_{\psi\circ q}(\u)\,\dot{\u} + \eta\, \J_{\psi\circ q}(\o)\H^{-1}_G(\u)\nabla L\!\circ\!q(\u)
    + \J_{\psi\circ q}(\o)\H^{-1}_G(\o)\J^\top_{\psi\circ q}(\u)\lam(t) = \zero\, .
\end{align*}
    The rest of the proof follows similarly by solving for $\lam(t)$ and rearranging the terms. Finally, applying the results of Theorem~\ref{thm:repam} concludes the proof.
\end{proof}

\section{Discretized Updates}
\label{app:discr}

In this section, we discuss different strategies for discretizing the CMD updates and provide examples for each case. 

The most straight-forward discretization of the unconstrained CMD update~\eqref{eq:cont-md} is the forward Euler (i.e. explicit) discretization, given in~\eqref{eq:md-explicit}. Note that this corresponds to a (approximate) minimizer of the discretized form of~\eqref{eq:main-obj}, that is,
\[
\argmin_{\o}\,\Big\{ \sfrac{1}{\eta}\,\big(D_F(\o, \o_s) - \underbrace{D_F(\o_s, \o_s)}_{=0}\big) + L(\o)\Big\}\, .
\]
An alternative way of discretizing is to apply the approximation on the equivalent natural gradient form~\eqref{eq:w-natural}, which yields
\[
\o_{s+1} - \o_s = -\eta\, \H_{F}^{-1}(\o_s)\,\nabla L(\o_s)\, .
\]
Despite being equivalent in continuous-time, the two approximations may correspond to different updates after discretization. As an example, for the EG update motivated by $f(\o) = \log \o$ link, the latter approximation yields
\[
\o_{s+1} = \o_s \odot \big(\one - \eta\, \nabla L(\o_s)\big)\, ,
\]
which corresponds to the unnormalized {\it prod} update, introduced by~\cite{prod} as a Taylor approximation of the original EG update.

The situation becomes more involved for discretizing the constrained updates. As the first approach, it is possible to directly discretize the projected CMD update~\eqref{eq:proj-cont-md}
\[
f\big(\ot_{s+1}\big) - f\big(\o_{s}\big)=-\eta\,\P_\psi(\o_s)\nabla L(\o_s)\, .
\]
 However, note that the new parameter $\ot_{w+1}$ may fall outside the constraint set~$\C_\psi \coloneqq \{\o\in \C|\,\psi\big(\o)\big) = \zero\}$. As a result, a Bregman projection~\citep{shay} into $\C_\psi$ may need to be applied after the update, that is
 \begin{equation}
\label{eq:bregman-proj}
\o_{s+1} = \argmin_{\o \in \C_\psi}\, D_F(\o, \ot_{s+1})\, .
 \end{equation}
As an example, for the normalized EG updates with the additional constraint that $\o^\top \one = 1$, we have $\P_\psi(\o) = \I_d - \one\o^\top$ and the approximation yields
\begin{align*}
\log\big(\ot_{s+1}\big) - \log\big(\o_{s}\big)
=-\eta\,\big(\nabla L(\o_s) - \one\,\mathbb{E}_{\o_s}[\nabla L(\o_s)] \big)\, ,
\end{align*}
where $\mathbb{E}_{\o_s}[\nabla L(\o_s)] = \o_s^\top \nabla L(\o_s)$. Clearly, $\ot_{s+1}$ may not necessarily satisfy $\ot_{s+1}^\top \one = 1$. Therefore, we apply
\[
\o_{s+1} = \frac{\ot_{s+1}}{\Vert \ot_{s+1}\Vert_1}\, ,
\]
which corresponds to the Bregman projection onto the unit simplex using the relative entropy divergence~\citep{eg}.

An alternative approach for discretizing the constrained update would be to first discretize the functional objective with the Lagrange multiplier~\eqref{eq:main-obj-lam} and then (approximately) solve for the update. That is,
\begin{align*}
\o_{s+1} = \argmin_{\o}\,\Big\{ \sfrac{1}{\eta}\,\big(D_F(\o, \o_s) - \underbrace{D_F(\o_s, \o_s)}_{=0}\big) + L(\o) + \lam^\top\psi(\o) \Big\}\, .
\end{align*}
Note that in this case, the update satisfies the constraint $\psi(\o_{s+1}) = \zero$ because of directly using the Lagrange multiplier. For the normalized EG update, this corresponds to the original normalized EG update in~\citep{wm},
\[
\o_{s+1} = \frac{\o_{s}\odot\exp\big(-\eta\, \nabla L(\o_s)\big)}{\Vert \o_{s}\odot\exp\big(-\eta\, \nabla L(\o_s)\big)\Vert_1}\, .
\]
Finally, it is also possible to discretized the projected natural gradient update~\eqref{eq:proj-cont-ngd}. Again, a Bregman projection into $\C_\psi$ may need to be required after the update, that is,
\[
\ot_{s+1} - \o_s = -\eta \P_\psi(\o_s)^\top\H_F^{-1}(\o_s)\nabla L(\o(t))\, ,
\]
followed by~\eqref{eq:bregman-proj}. For the normalized EG update, the first step corresponds to
\[
\o_{s+1} = \o_s \odot \Big(\one - \eta \big(\nabla L(\o_s) - \one\,\mathbb{E}_{\o_s}[\nabla L(\o_s)] \big)\Big)\, ,
\]
which recovers to the {\it approximated EG} update of~\cite{eg}. Note that $\o_{s+1}^\top \one = 1$ and therefore, no projection step is required in this case.

\end{document}